\documentclass[twoside,11pt]{article}
\usepackage[preprint]{jmlr2e}
\usepackage[utf8]{inputenc}
\usepackage{url}
\usepackage{booktabs}
\usepackage{amsmath,amssymb}
\usepackage{graphicx}
\usepackage{lastpage}
\usepackage{enumitem}
\usepackage{microtype}
\hypersetup{
  hidelinks,
  pdfauthor={Michael Timothy Bennett},
  pdftitle={Are Flat Minima an Illusion?}
}

\newcommand{\Ext}[1]{\mathrm{Ext}\!\left(#1\right)}
\newcommand{\R}{\mathbb{R}}
\newcommand{\1}{\mathbf 1}
\newcommand{\bkl}{\operatorname{kl}}
\renewenvironment{proof}[1][Proof]{\par\noindent\textbf{#1.}\ }{\hfill\BlackBox\par\vspace{2mm}}
\DeclareMathOperator*{\argmax}{arg\,max}

\newif\ifjready
\jreadyfalse
\newcommand{\mjr}{N/A}\newcommand{\mjci}{N/A}\newcommand{\mjq}{N/A}
\newcommand{\mhr}{N/A}\newcommand{\mhci}{N/A}\newcommand{\mhq}{N/A}
\newcommand{\mfr}{N/A}\newcommand{\mfci}{N/A}\newcommand{\mfq}{N/A}
\newcommand{\mmr}{N/A}\newcommand{\mmci}{N/A}\newcommand{\mmq}{N/A}
\newcommand{\mwr}{N/A}\newcommand{\mwci}{N/A}\newcommand{\mwq}{N/A}
\newcommand{\fjr}{N/A}\newcommand{\fjci}{N/A}\newcommand{\fjq}{N/A}
\newcommand{\fhr}{N/A}\newcommand{\fhci}{N/A}\newcommand{\fhq}{N/A}
\newcommand{\ffr}{N/A}\newcommand{\ffci}{N/A}\newcommand{\ffq}{N/A}
\newcommand{\fmr}{N/A}\newcommand{\fmci}{N/A}\newcommand{\fmq}{N/A}
\newcommand{\fwr}{N/A}\newcommand{\fwci}{N/A}\newcommand{\fwq}{N/A}
\newcommand{\mrel}{N/A}\newcommand{\frel}{N/A}
\newcommand{\rrho}{N/A}\newcommand{\rci}{N/A}\newcommand{\rrq}{N/A}
\newcommand{\jmnmean}{N/A}\newcommand{\jrdmean}{N/A}\newcommand{\rdiff}{N/A}\newcommand{\rdq}{N/A}
\jreadytrue
\renewcommand{\mjr}{+0.290}
\renewcommand{\mjci}{[$+0.094$, $+0.468$]}
\renewcommand{\mjq}{0.004}
\renewcommand{\mhr}{+0.091}
\renewcommand{\mhci}{[$-0.123$, $+0.298$]}
\renewcommand{\mhq}{1.000}
\renewcommand{\mfr}{+0.168}
\renewcommand{\mfci}{[$-0.032$, $+0.361$]}
\renewcommand{\mfq}{0.389}
\renewcommand{\mmr}{+0.504}
\renewcommand{\mmci}{[$+0.329$, $+0.652$]}
\renewcommand{\mmq}{1.60\times 10^{-4}}
\renewcommand{\mwr}{+0.398}
\renewcommand{\mwci}{[$+0.217$, $+0.565$]}
\renewcommand{\mwq}{2.40\times 10^{-4}}
\renewcommand{\fjr}{+0.468}
\renewcommand{\fjci}{[$+0.301$, $+0.607$]}
\renewcommand{\fjq}{4.00\times 10^{-5}}
\renewcommand{\fhr}{+0.056}
\renewcommand{\fhci}{[$-0.142$, $+0.252$]}
\renewcommand{\fhq}{1.000}
\renewcommand{\ffr}{-0.039}
\renewcommand{\ffci}{[$-0.235$, $+0.164$]}
\renewcommand{\ffq}{1.000}
\renewcommand{\fmr}{+0.546}
\renewcommand{\fmci}{[$+0.381$, $+0.680$]}
\renewcommand{\fmq}{1.60\times 10^{-4}}
\renewcommand{\fwr}{+0.361}
\renewcommand{\fwci}{[$+0.174$, $+0.524$]}
\renewcommand{\fwq}{1.40\times 10^{-3}}
\renewcommand{\mrel}{0.940}
\renewcommand{\frel}{0.956}
\renewcommand{\rrho}{+0.058}
\renewcommand{\rci}{[$-0.146$, $+0.263$]}
\renewcommand{\rrq}{0.567}
\renewcommand{\jmnmean}{0.287}
\renewcommand{\jrdmean}{0.420}
\renewcommand{\rdiff}{-0.134}
\renewcommand{\rdq}{2.00\times 10^{-5}}
\newcommand{\mpemp}{0.434}
\newcommand{\mplb}{0.241}
\newcommand{\mpseed}{68}
\newcommand{\fpemp}{0.467}
\newcommand{\fplb}{0.269}
\newcommand{\fpseed}{19}
\newcommand{\pacrhs}{0.0892}

\jmlrheading{27}{2026}{1--\pageref{LastPage}}{Submitted 7/26}{}{M000}{Michael Timothy Bennett}
\ShortHeadings{Weakness and Neural Generalisation}{Bennett}
\firstpageno{1}

\title{Are Flat Minima an Illusion?}
\author{\name Michael Timothy Bennett \email m@michaeltimothybennett.com\\
\addr School of Computing, The Australian National University, Canberra, ACT 2601, Australia}
\editor{To be assigned}
\begin{document}
\maketitle

\begin{abstract}
Flat minima are an account of why deep networks generalise. However flatness is a matter of form (parameters), while generalisation is of function. The same function can be a result of many different parameterisations. I demonstrate this by rescaling ReLU networks, changing raw Hessian trace by up to $99$ times while every prediction remains fixed. Raw curvature cannot identify a function-level explanation. Previous theoretical work traced generalisation to the weakness of constraints implied by function, meaning the freedom a model retains within the bounds of what it has learned to be correct. A policy is weaker when more future commitments remain compatible with what it has learned, allowing more freedom to adapt. To measure this for neural networks, I freeze the last hidden representation and ask whether each of 512 sampled label bundles can be met by a replacement affine classifier. The resulting joint completion score is invariant under invertible linear mixing and translation of feature coordinates. Across two predeclared cohorts of 100 networks, it predicts held-out accuracy with rank correlations $0.29$ and $0.47$. Raw Hessian trace and relative flatness have no multiplicity-corrected association. To put it provocatively, freedom is correlated with adaptability, while flatness is a matter of description.
\end{abstract}

\section{Introduction}\label{sec:intro}

Networks in low-curvature parameter regions often generalise better than networks in high-curvature regions \citep{hochreiter1997flat,keskar2017}. Curvature is measured from the Hessian, the matrix of second derivatives of training loss with respect to model parameters. Its trace sums curvature along the chosen parameter coordinates. Sharpness-Aware Minimisation trains against losses in a small parameter neighbourhood and often improves accuracy \citep{foret2021sharpnessaware}. These findings establish an engineering effect. The function-level explanation remains open.

For networks whose activations scale linearly under positive rescaling, including rectified linear unit (ReLU) networks, a function-preserving layer rescaling changes raw Hessian geometry while every prediction stays fixed \citep{dinh2017sharp}. Gradient-based training may favour a restricted part of each rescaling orbit, so fixed-coordinate curvature can record how a function was reached. Reparameterisation-adjusted measures were developed for this reason \citep{tsuzuku2020normalized,petzka2021relative,kwon2021asam}. The target here is the function-level question. One function admits many raw traces. Raw trace therefore characterises one parameter vector among many that represent the same function.

Stack Theory proposes another object. An embodiment is the vocabulary of constraints a system can physically express and the language generated by that vocabulary. A task $\alpha$ has a policy version space $\Pi_\alpha$ containing its correct policies \citep{mitchell1982}. Weakness assigns each member a score equal to the embodied commitments that remain compatible after correctness has been imposed \citep{bennett2023b,bennett2024b,bennett2025d,bennett2025thesis}. Within one embodiment, Bagi is the rule that chooses a weakest member of $\Pi_\alpha$. Across embodiments, Unagi is the rule that compares distinct unseen external facts expressed by the languages \citep{bennett2026wrong}.

This paper uses \emph{free-maxing}, the current name for the earlier \emph{w-maxing} meta-approach. Section~\ref{sec:related} explains the rename.

This paper applies those definitions to neural networks. The finite weakest-correct theorem and its continuous extension are prior Stack Theory results \citep{bennett2023b,bennett2025d,bennett2025thesis,bennett2026wrong}. The new contributions are as follows.

\begin{enumerate}[leftmargin=*]
\item A PAC-Bayes bound that combines observed future-demand survival with prior-to-posterior information cost, together with an identity for the minimum information cost of restricting an unseen-region measure to one policy buffer.
\item A neural construction that treats class assignments as constraints on a frozen feature representation, together with a fixed joint-bank score and a theorem proving invariance under reversible linear mixing and translation of features.
\item A task-alignment theorem giving the exact anchor-agreement condition for one common nonempty task.
\item Three predeclared 100-network cohorts whose training-side measurements were sealed before test labels were loaded. Comparators are raw Hessian trace, scale-adjusted relative flatness, a class-separation score divided by layer spectral norms, and weight norm.
\item A rescaling experiment in which raw Hessian trace changes by up to $99\times$ while every prediction remains fixed.
\end{enumerate}

The neural construction and completed cohort results are first presented here. The finite weakest-correct theorem remains prior background. A two-page IJCAI summary reported the finite theory and earlier binary-arithmetic experiments \citep{bennett2025e}.

The neural evidence has a declared scope. The joint bank score correlates with held-out accuracy in both natural-label cohorts. Spectral-product margin has larger point estimates, and smaller weight norm also predicts accuracy. Random-label representations leave more of the uniform bank feasible while their score-accuracy association centres near zero. These findings delimit the estimator. The Stack Theory theorem retains its scope over compatible future mass under a fixed task, embodied language, and demand law.

\section{Stack Theory}\label{sec:theory}

The paper uses the definitions in the Stack Theory book and technical appendix \citep{bennett2025thesis}. A program here means a physically embodied binary constraint.

\begin{definition}[Embodied language and weakness]\label{def:lang}
An environment is a nonempty set $\Phi$ of mutually exclusive states. A program is any set $p\subseteq\Phi$, and $\mathcal P:=2^\Phi$ is the set of all programs. A vocabulary $\mathfrak v\subseteq\mathcal P$ induces the embodied language
\[
L_{\mathfrak v}
=
\left\{l\subseteq\mathfrak v\;\middle|\;l\text{ is finite and }\bigcap_{p\in l}p\neq\emptyset\right\}.
\]
Its elements are statements. Their truth sets are $T(l):=\bigcap_{p\in l}p$, with $T(\emptyset):=\Phi$. A completion of $l$ is any $y\in L_{\mathfrak v}$ with $l\subseteq y$, and
\[
\Ext{l}:=\{y\in L_{\mathfrak v}:l\subseteq y\}.
\]
For $X\subseteq L_{\mathfrak v}$, write $\Ext{X}:=\bigcup_{x\in X}\Ext{x}$. In a finite language, the weakness of $l$ is
\[
w(l):=|\Ext{l}|.
\]
\end{definition}

A statement is a satisfiable bundle of commitments one body can instantiate together. Weakness counts the further commitments which remain compatible. Candidate-policy count and description length are separate quantities.

\begin{definition}[Tasks, policies, and the version space]\label{def:task}
A $\mathfrak v$-task is a pair $\alpha=\langle I_\alpha,O_\alpha\rangle$ with
\[
I_\alpha\subseteq L_{\mathfrak v},
\qquad
O_\alpha\subseteq\Ext{I_\alpha}.
\]
A policy is a statement $\pi\in L_{\mathfrak v}$. It is correct for $\alpha$ when
\[
\Ext{I_\alpha}\cap\Ext{\pi}=O_\alpha.
\]
The correct-policy version space is
\[
\Pi_\alpha
=
\left\{\pi\in L_{\mathfrak v}:
\Ext{I_\alpha}\cap\Ext{\pi}=O_\alpha
\right\}.
\]
When $\alpha$ is the observed task from which a learner must generalise, it is called the child task and $\Pi_\alpha$ is its child-correct policy version space.
\end{definition}

Bagi selects
\[
\pi^\star\in\argmax_{\pi\in\Pi_\alpha}w(\pi).
\]
Weakness assigns one score to each member of $\Pi_\alpha$. The quantity $|\Pi_\alpha|$ counts the whole version space.

\begin{definition}[Measured weakness and future demands]\label{def:meas}
A measured embodied language is $(\mathfrak v,\mathcal A_L,\mu)$. Here $\mathcal A_L$ is a $\sigma$-algebra, meaning a collection of subsets of $L_{\mathfrak v}$ closed under complements and countable unions. Its members are the measurable regions, and $\mu$ assigns their sizes. Every extension used below belongs to $\mathcal A_L$. The measure is $\sigma$-finite, so $L_{\mathfrak v}$ can be covered by countably many measurable regions of finite measure. Define
\[
w_\mu(l):=\mu(\Ext{l}).
\]
For a task $\alpha$, define its unseen region by
\[
U:=L_{\mathfrak v}\setminus\Ext{I_\alpha}.
\]
Fix a correct policy $\pi\in\Pi_\alpha$. Its buffer is
\[
B_\pi:=\Ext{\pi}\cap U.
\]
Assume $U$ is measurable and $0<\mu(U)<\infty$. Let $S\subseteq U$ be a random measurable set of future demands with probability law $P$. Every containment event used below is assumed measurable. Generalisation probability is
\[
G(\pi,P):=P(S\subseteq B_\pi).
\]
A demand law is $\mu$-exchangeable when policies with equally measured buffers have equal survival probability and larger buffer measure never lowers that probability.
\end{definition}

For every correct policy,
\[
\Ext{\pi}=O_\alpha\sqcup B_\pi,
\]
where $\sqcup$ denotes a disjoint union.
When $\mu(O_\alpha)<\infty$, $\mu$-weakness and buffer measure rank correct policies identically.

\begin{theorem}[Weakest-correct optimality]\label{thm:dom}
Fix a measured embodied language and a task $\alpha$ with $\mu(O_\alpha)<\infty$. Under any $\mu$-exchangeable future-demand law,
\[
\mu(B_{\pi_1})\geq\mu(B_{\pi_2})
\quad\Longrightarrow\quad
G(\pi_1,P)\geq G(\pi_2,P)
\]
for all $\pi_1,\pi_2\in\Pi_\alpha$. If generalisation probability is strictly increasing across the buffer measures attained by $\Pi_\alpha$, strict inequality of buffer measure implies strict inequality of generalisation probability.
\end{theorem}

This is the continuous Stack Theory rule \citep{bennett2025thesis,bennett2026wrong}. It is universal over declared embodiments. For every fixed $(\mathfrak v,\mu)$ and matching demand law, select a correct policy that maximises $w_\mu$. Numerical comparisons use matched vocabularies and measures.

Equal-measure symmetry is a sufficient premise. It requires $P(S\subseteq A)=P(S\subseteq B)$ whenever measurable regions $A,B\subseteq U$ have $\mu(A)=\mu(B)$. An atomless region can be divided to obtain a measurable subset of any smaller measure. On such an unseen region, corpus $\mu$-exchangeability follows from this symmetry and event containment. The derivation is given by the argument in Appendix~\ref{app:proofs}.

For finite $U=\{u_1,\ldots,u_m\}$, suppose each unseen requirement $u_i$ becomes relevant independently with probability $0\leq r_i<1$. Then
\[
G(\pi,P)=\prod_{u_i\notin B_\pi}(1-r_i).
\]
Maximising $G(\pi,P)$ is equivalent to maximising
\[
\sum_{u_i\in B_\pi}-\log(1-r_i).
\]
Equal positive rates recover ordinary weakness up to a positive factor. Known demand structure changes the weights, while compatible future mass remains the target.

Bagi operates inside one embodied language. Semantic Unagi compares different embodied languages by matching external truth conditions that mean the same thing in each \citep{bennett2026wrong}. The neural protocol below uses $J_H$ until its network policies have been shown to belong to one common policy version space.

\section{Raw Curvature Varies Within One Function}\label{sec:curvature}

Fix a neural-network architecture. Let $\Theta$ be its parameter space, meaning the set of all allowed weight and bias values, and let $\theta\in\Theta$ denote one parameter vector. Write $F(\theta)$ for the input-output function represented by $\theta$. Treat each representable function as one state and define the environment
\[
\Phi:=\{F(\theta):\theta\in\Theta\}.
\]
Thus $F:\Theta\to\Phi$. Fix a vocabulary $\mathfrak v\subseteq2^\Phi$, so each $p\in\mathfrak v$ is a yes-or-no property of represented functions. The Stack Theory encoding map is
\[
\operatorname{Enc}_{\mathfrak v}(F(\theta))
:=
\{p\in\mathfrak v:F(\theta)\in p\}.
\]
This statement records every property in $\mathfrak v$ that is true of the represented function. Parameter vectors representing the same function therefore receive the same encoding.

A function-preserving reparameterisation is a smooth, reversible change of parameter coordinates $\psi:\Theta\to\Theta$ that satisfies $F(\psi(\theta))=F(\theta)$.

\begin{proposition}[Function-level invariance]\label{prop:funcinv}
If $\psi$ is function preserving, then
\[
\operatorname{Enc}_{\mathfrak v}(F(\psi(\theta)))
=
\operatorname{Enc}_{\mathfrak v}(F(\theta)).
\]
The encoded statement is identical. Its extension and finite weakness are therefore unchanged. Under a fixed measured language, its $\mu$-weakness is unchanged. When this statement is a policy for the same task and demand law, its buffer and $G$ are unchanged.
\end{proposition}

For a scalar preactivation $z\in\R$, a ReLU activation obeys $\operatorname{ReLU}(\beta z)=\beta\operatorname{ReLU}(z)$ for $\beta>0$. For two consecutive ReLU layers, multiply one layer and its bias by $\beta$ and divide the next weight matrix by $\beta$. This positive-homogeneity identity preserves every output. Raw Hessian trace changes under this coordinate map \citep{dinh2017sharp}. In the retained experiment, the trace changes by up to $99\times$ while every test prediction remains fixed.

These observations establish coordinate variability within one learned function. A rescaling orbit is the family of parameter vectors connected by function-preserving rescalings. Gradient flow, a continuous-time idealisation of gradient descent, can constrain which members of that family training visits \citep{du2018balanced}. Fixed-coordinate curvature can then describe an optimiser-conditioned path. Adjusted measures can also remain unchanged under specified classes of coordinate transformation \citep{tsuzuku2020normalized,petzka2021relative,kwon2021asam}. The experiments therefore include relative flatness alongside raw trace.

\section{Joint Neural Completion}\label{sec:joint}

A feature vector is the finite list of numbers produced for one input by a hidden network layer. A feature-classifier vocabulary describes which labels can be assigned after a network converts inputs into these vectors. Each vocabulary item is one yes-or-no claim. It asks whether some final affine classifier can assign class $c$ to input $x$. An affine classifier computes one score per class by multiplying a feature vector by a weight matrix and adding a bias. A neural construction is a mathematical specification of inputs, feature vectors, allowed classifiers, and class claims. The construction below creates such a vocabulary for each trained network. It freezes output from the layer immediately before its final classifier and allows every affine replacement classifier. The symbol $s^{(f)}$ translates an external probe-label demand into network $f$'s vocabulary. The statistic $J_H$ records how often sampled joint demands remain feasible with that network's empirical policy.

\subsection{Feature-classifier vocabulary}

Fix a trained network $f$ and a finite evaluation set $X$. Let $d\geq1$ be the feature dimension. Network $f$ maps each input $x\in X$ to a feature vector $\varphi_f(x)\in\R^d$. Let $K\geq2$ be the number of classes and let $\mathcal C=\{0,\ldots,K-1\}$ be the class set. A replacement affine head has a weight matrix $W\in\R^{K\times d}$ and a bias vector $b\in\R^K$. It assigns one score $W_c\!\cdot\!\varphi_f(x)+b_c$ to each class $c$. Take the environment to be the set of all such replacement heads, written $\Phi:=\R^{Kd+K}$.
For each $x\in X$ and $c\in\mathcal C$, define
\[
p^{\varphi_f}_{x,c}
:=
\left\{(W,b)\in\Phi:
W_c\!\cdot\!\varphi_f(x)+b_c>
W_j\!\cdot\!\varphi_f(x)+b_j
\text{ for every }j\in\mathcal C\setminus\{c\}
\right\}.
\]
The vocabulary is
\[
\mathfrak v_{\varphi_f}
:=
\{p^{\varphi_f}_{x,c}:x\in X,\ c\in\mathcal C\}.
\]
A bundle belongs to $L_{\mathfrak v_{\varphi_f}}$ exactly when one affine head satisfies all class inequalities at once. Checking this means asking whether a finite system of linear inequalities has a solution. Standard linear-programming software performs that check. I abbreviate this as an LP feasibility check.

Partition $X$ into disjoint child, anchor, and probe sets. Child inputs carry observed labels used to define the fitted task. Anchor inputs preserve selected predictions made by the trained network and thereby retain a finite part of its learned policy. Probe inputs receive hypothetical labels sampled after training and test how much further commitment remains feasible. Write these sets as $X_{\mathrm{child}}$, $X_{\mathrm{anchor}}$, and $X_{\mathrm{probe}}$. Let $y_{\mathrm{child}}:X_{\mathrm{child}}\to\mathcal C$ be the observed child-label map, and let $\widehat y_f(a)$ be network $f$'s predicted class on anchor $a$. Define its empirical statement
\[
\pi_f
=
\{p^{\varphi_f}_{x,y_{\mathrm{child}}(x)}:x\in X_{\mathrm{child}}\}
\cup
\{p^{\varphi_f}_{a,\widehat y_f(a)}:a\in X_{\mathrm{anchor}}\}.
\]
Child terms preserve observed labels. Anchor terms preserve a finite reading of the trained network's policy. The statement is empirical because both parts come from a finite observed sample or from fixed network predictions. Resulting feasibility scores concern the frozen representation and every affine classifier that can read it.

A one-dimensional example shows what joint feasibility means. Let $K=2$ and define the score difference
\[
g(z):=(W_1-W_0)z+(b_1-b_0).
\]
A base point at $z=1$ labelled class $1$ requires $g(1)\geq1$. Adding a probe at $z=-1$ labelled class $0$ requires $g(-1)\leq-1$. Both hold for $g(z)=z$. Assigning class $0$ at the same point $z=1$ creates the contradictory requirements $g(1)\geq1$ and $g(1)\leq-1$. Every bank entry asks this type of question for several probes at once under one shared head.

\subsection{The common-task condition}

Each trained network has a different feature map and therefore a different feature-classifier vocabulary. Bagi compares policies inside one vocabulary. Comparing networks therefore uses the cross-vocabulary procedure called semantic Unagi. Start with an external task stated directly as input-label requirements, before any network-specific features are introduced. Then ask whether each network-specific language preserves those same requirements. This shared template is called an uninstantiated task until a particular feature language is used to express it.

The required compatibility condition follows from the corpus definitions.

\begin{proposition}[Common-task compatibility]\label{prop:common}
Let $M\geq1$ and let $\pi_1,\ldots,\pi_M\in L_{\mathfrak v}$.

\begin{enumerate}[label=(\arabic*)]
\item If a $\mathfrak v$-task $\alpha$ has $O_\alpha\neq\emptyset$ and $\pi_i\in\Pi_\alpha$ for every $i$, then
\[
\bigcup_{i=1}^{M}\pi_i\in L_{\mathfrak v}.
\]
\item Conversely, if $\bigcup_i\pi_i\in L_{\mathfrak v}$, then there exists a $\mathfrak v$-task with nonempty correct outputs for which every $\pi_i$ is correct.
\end{enumerate}
\end{proposition}

Apply the proposition in the common external language of partial label assignments. Take the environment to be the set of total labellings $g:X\to\mathcal C$. The program for input $x$ and class $c$ is the set of total labellings with $g(x)=c$. A statement is therefore a partial labelling, and it is satisfiable exactly when it assigns at most one class to each input.

\begin{corollary}[Anchor agreement obstruction]\label{cor:anchor}
Suppose several external empirical policies contain the same child labels and one class label for every anchor. They are compatible with some nonempty common task if and only if they assign the same class to every anchor.
\end{corollary}

Each trained network supplies its own anchor labels. Common-task compatibility therefore requires anchor agreement. One disagreement creates contradictory class assignments and empties the common truth set. Agreement proves compatibility with some common task. Correctness for the observed child task and preservation of all truth conditions require separate task-level proofs.

The completed aggregate archive stores cohort summaries. Updated notebooks retain each network's list of anchor predictions. An anchor-audit notebook can recover those lists from saved checkpoints while reusing existing linear-programming records. Auditing completed cohorts therefore requires saved checkpoints. A common-task interpretation additionally requires anchor agreement.

\subsection{Matched joint bank}

A bank is a fixed list of sampled future demands used for every network. Let $H=(s_1,\ldots,s_m)$ be a finite list of partial label assignments on $X_{\mathrm{probe}}$. A partial assignment attaches classes to some probe inputs and leaves the others unspecified. Writing $H$ as an indexed list retains repeated draws instead of merging equal entries. For each entry, its reading in network $f$ is
\[
s_j^{(f)}
:=
\{p^{\varphi_f}_{x,c}:(x,c)\in s_j\}.
\]
Define
\[
J_H(f)
:=
\frac1m
\sum_{j=1}^{m}
\1\{\pi_f\cup s_j^{(f)}\in L_{\mathfrak v_{\varphi_f}}\}.
\]
Here $\1\{\cdot\}$ is the indicator. It equals one when the condition inside braces holds and zero otherwise.

\begin{proposition}[Finite-bank survival]\label{prop:bank}
Under a uniform draw from the indexed bank, $J_H(f)$ is exactly the empirical probability that the drawn external demand is jointly feasible with $\pi_f$. In the implemented design, each path is an independent sequence of sampled probe-label commitments. Four fixed prefixes are averaged within each path. Hence the expected value of $J_H(f)$ equals the expected path score under the declared probability rule for sampling paths.
\end{proposition}

Every successful union is a member of $\Ext{\pi_f}$. Thus $J_H$ estimates external-demand survival on the indexed bank within each feature-classifier language.

\subsection{Affine invariance}

\begin{theorem}[Affine feature invariance]\label{thm:affine}
Let $A\in\R^{d\times d}$ be invertible and let $a\in\R^d$. Define $\varphi'_f(x)=A\varphi_f(x)+a$. Matrix $A$ mixes feature coordinates without losing information, while $a$ translates them. Construct the transformed vocabulary, empirical statement, and external readings from $\varphi'_f$. Under
\[
W'=WA^{-1},
\qquad
b'=b-WA^{-1}a,
\]
for every external bundle $s$,
\[
\pi_f\cup s^{(f)}\in L_{\mathfrak v_{\varphi_f}}
\quad\Longleftrightarrow\quad
\pi'_f\cup s^{(f')}
\in L_{\mathfrak v_{\varphi'_f}}.
\]
Hence $J_H$ is unchanged.
\end{theorem}

The replacement head ranges over all affine parameters. The implementation uses margin one, meaning a required class score must exceed every competing score by at least one. Every strictly feasible statement can be positively rescaled to meet any fixed positive margin.

Freezing the feature map fixes the abstraction layer at which completion feasibility is measured. Penultimate means the layer immediately before the output classifier. The experiment uses that layer, 32 anchors, four commitment orders corresponding to $2,4,8,$ and $16$ simultaneous probe labels, equal order weights, and a uniform rule for sampling probe labels. All choices were fixed before outcome access. The evolutionary sequel learns which abstraction layers and weights best predict validation performance on separate cohorts.
\section{Experiments}\label{sec:experiments}

\subsection{Function-preserving rescaling}

The retained rescaling experiment uses three-layer multilayer perceptrons, meaning fully connected feed-forward networks, with ReLU activations and two positive layer-rescaling maps. Across twelve trained networks, the largest raw Hessian-trace ratio is $99\times$. Every output and test prediction is preserved within numerical precision. Appendix~\ref{app:rescale} reports representative values. This test establishes coordinate variability within one learned function.

\subsection{Cohorts}

The joint experiment trains two cohorts, meaning two groups of networks evaluated under matched protocols. One contains 100 networks trained on MNIST, a handwritten-digit dataset, and the other contains 100 trained on Fashion-MNIST, a clothing-image dataset \citep{xiao2017fashionmnist}. Every network has architecture $784\to64\to8\to10$. Within each cohort, the child set and training protocol are fixed. Only the training seed changes. A seed controls pseudorandom parameter initialisation and the order of small training batches, called minibatches. The analysis therefore concerns model variation conditional on one sample and one algorithm. Algorithm-level performance would additionally average over training samples and training randomness.

A fixed permutation of the training partition allocates 250 child inputs, 32 anchors, and 128 free probes. In the supplied notebook, the official test partition is loaded only after all 100 rows of training-side measurements have been computed. A cryptographic hash records a digital fingerprint of those rows before test access. The estimator samples candidate labels for probe inputs, while their dataset labels remain sealed.

The bank contains 128 nested paths. Each path samples 16 probe-label commitments, and every longer prefix contains the shorter one. Prefixes containing $2,4,8,$ and $16$ simultaneous commitments produce 512 joint statements per network. One shared affine head must satisfy each statement together with the empirical policy. The nested design permits early termination of each path. Odd and even paths form disjoint half-banks for reliability.

Comparators are raw Hessian trace, final-layer relative flatness after feature-variance normalisation \citep{petzka2021relative}, spectral-product margin motivated by spectral margin theory \citep{bartlett2017spectral}, and parameter $\ell_2$ norm. Raw Hessian trace sums coordinate curvature. Relative flatness adjusts final-layer curvature for weight and feature scale. Spectral-product margin divides observed class-score separation by a product of layer spectral norms. A spectral norm is the largest singular value of a weight matrix. Parameter $\ell_2$ norm is the Euclidean size of all parameters. Scores are oriented so that a larger value predicts greater accuracy. Raw Hessian trace, relative flatness, and weight norm are negated. Joint bank score and margin retain their sign.

Spearman's $\rho$ measures whether two network rankings rise together. A 50,000-draw permutation test estimates how often an association at least this large appears after accuracy ranks are randomly reassigned. Network-level percentile bootstrap intervals use 10,000 resamples of networks. Holm correction controls family-wise error, meaning the chance of at least one false rejection in a family, across the two natural-label joint tests and separately across eight baseline tests. Paired differences between correlations from the same networks use Bonferroni $98.75\%$ bootstrap intervals within each dataset. The reducer performs these calculations using predeclared equal weights for the four commitment orders. CI abbreviates confidence interval. Each reported $q$ value is a $p$ value after Holm correction.

\ifjready
\begin{table}[t]
\centering
\small
\caption{MNIST cohort. Larger oriented scores predict greater test accuracy. Split-bank Spearman-Brown reliability is $\mrel$.}
\label{tab:mnres}
\begin{tabular}{@{}lccc@{}}
\toprule
Measure & Oriented $\rho$ & 95\% CI & Holm $q$ \\
\midrule
Joint bank score $J_H$ & $\mjr$ & \mjci & $\mjq$ \\
Raw Hessian trace & $\mhr$ & \mhci & $\mhq$ \\
Final-layer relative flatness & $\mfr$ & \mfci & $\mfq$ \\
Spectral-product margin & $\mmr$ & \mmci & $\mmq$ \\
Weight norm & $\mwr$ & \mwci & $\mwq$ \\
\bottomrule
\end{tabular}
\end{table}

\begin{table}[t]
\centering
\small
\caption{Fashion-MNIST cohort. Larger oriented scores predict greater test accuracy. Split-bank Spearman-Brown reliability is $\frel$.}
\label{tab:fmres}
\begin{tabular}{@{}lccc@{}}
\toprule
Measure & Oriented $\rho$ & 95\% CI & Holm $q$ \\
\midrule
Joint bank score $J_H$ & $\fjr$ & \fjci & $\fjq$ \\
Raw Hessian trace & $\fhr$ & \fhci & $\fhq$ \\
Final-layer relative flatness & $\ffr$ & \ffci & $\ffq$ \\
Spectral-product margin & $\fmr$ & \fmci & $\fmq$ \\
Weight norm & $\fwr$ & \fwci & $\fwq$ \\
\bottomrule
\end{tabular}
\end{table}

The joint bank score has a positive association in both natural-label cohorts. Confidence intervals exclude zero and both tests pass the primary Holm correction. Split-bank Spearman-Brown reliability compares scores from odd and even path halves, then adjusts their agreement to the full bank length. It is $\mrel$ on MNIST and $\frel$ on Fashion-MNIST. Linear-programming lower and upper bounds count any unresolved solver outcome as infeasible or feasible. They produce identical correlations here because every solver outcome is resolved.

Spectral-product margin has the largest point estimate in both cohorts. Smaller weight norm also has a corrected positive association. Paired intervals establish a greater joint-bank correlation than raw Hessian trace and relative flatness on Fashion-MNIST. The remaining paired intervals include zero.

The 128 paths are independent draws from the declared path-sampling rule. For path $S_i$, define $u(f,S_i)$ as the mean feasibility of its four nested prefixes. Then $J_H(f)=128^{-1}\sum_i u(f,S_i)$. Conditioned on the trained cohort and fixed inputs, a uniform prior over the 100 frozen networks is independent of the sampled paths. Section~\ref{sec:pac} applies PAC-Bayes to a point posterior, meaning a distribution that places all probability on the network selected by bank score before test accuracy is loaded.

\begin{table}[t]
\centering
\small
\caption{PAC-Bayes certificates for the bank-selected network. Each cohort uses $n=128$, a uniform prior over 100 frozen networks, and $\delta=0.025$. The two bounds therefore hold together with probability at least $0.95$. Each bound concerns expected feasibility over future paths drawn by the same sampling rule.}
\label{tab:pac}
\begin{tabular}{@{}lccc@{}}
\toprule
Cohort & Selected seed & Bank score $J_H$ & Expected path-score lower bound \\
\midrule
MNIST & $\mpseed$ & $\mpemp$ & $\mplb$ \\
Fashion-MNIST & $\fpseed$ & $\fpemp$ & $\fplb$ \\
\bottomrule
\end{tabular}
\end{table}

For either cohort, the right-hand side of the binary KL inequality is $\pacrhs$. Positive lower bounds certify expected external-demand survival under the full path-sampling rule rather than only on the 128 observed paths. Classification-error certification remains a separate target.
\else
The cohort results are pending.
\fi

\subsection{Random labels}

A third 100-network MNIST cohort uses one fixed random relabelling of the training partition and a separate fixed relabelling of the test partition. Input splits, architecture, optimiser, bank, and seeds match the natural-label cohort.

\ifjready
The random-label association is $\rho=\rrho$ with 95\% CI \rci{} and Holm $q=\rrq$. Mean $J_H$ is $\jmnmean$ under natural labels and $\jrdmean$ under random labels. Their seed-paired natural-minus-random difference is $\rdiff$ with Holm $q=\rdq$.
\fi

Random-label training leaves more of the uniform joint bank feasible on average, while its association with held-out random-label accuracy centres near zero. This result bounds the bank estimator. The task, learned policy, representation, and relevant future-demand law all changed. The control therefore concerns cross-task behaviour under a uniform arbitrary-label law. The estimator has a task-specific interpretation.

\begin{figure}[t]
\centering
\includegraphics[width=0.48\textwidth]{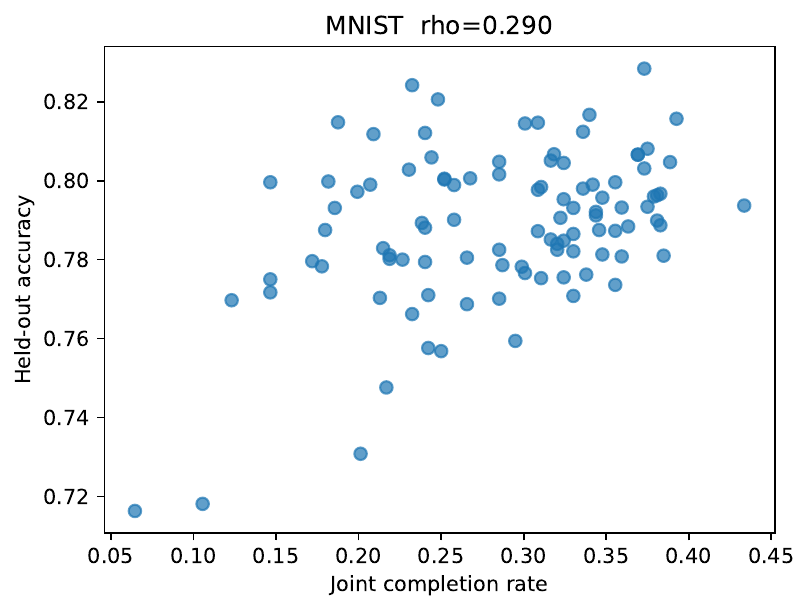}\hfill
\includegraphics[width=0.48\textwidth]{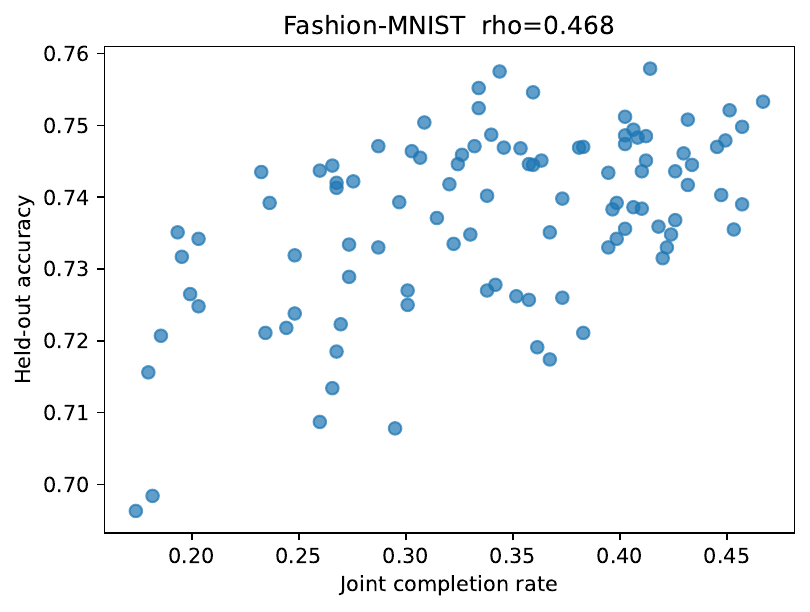}\\[0.5ex]
\includegraphics[width=0.48\textwidth]{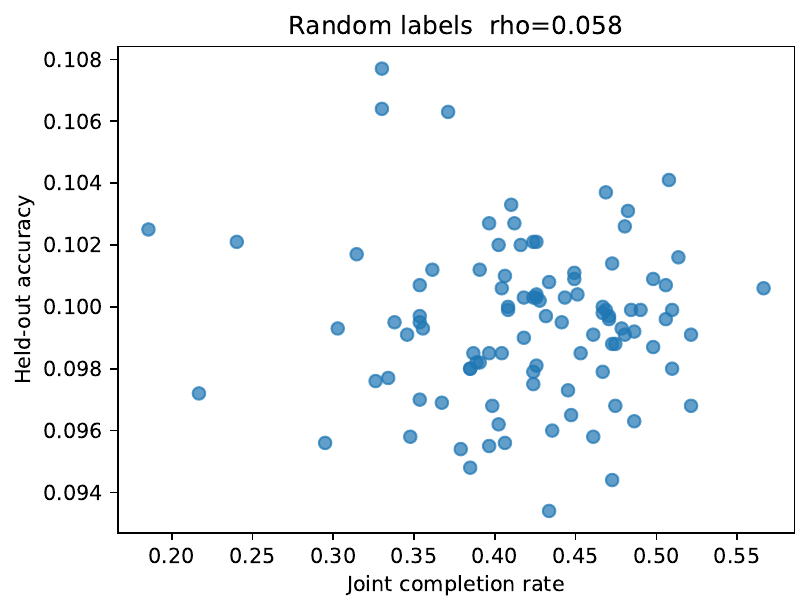}
\caption{Joint bank score and held-out accuracy. Natural-label cohorts show positive ranking associations. The random-label association centres near zero despite greater mean bank score.}
\label{fig:scatter}
\end{figure}

\section{PAC-Bayes Certificates}\label{sec:pac}

PAC-Bayes combines Probably Approximately Correct learning bounds with Bayesian priors and posteriors. It bounds expected future performance of a distribution over candidates. A prior is chosen before sampled future demands are observed. A posterior may be chosen afterward. Kullback-Leibler divergence, abbreviated KL, measures how much the posterior departs from the prior. Here candidates are child-correct policies and performance is survival of future demands. Let $S_1,\ldots,S_n$ be independent future-demand sets drawn from $P$. Let $Q_0$ be a prior on $\Pi_\alpha$ fixed before these demands are drawn, and let $Q$ be any posterior on $\Pi_\alpha$. Define
\[
\widehat G_n(Q)
:=
\frac1n\sum_{i=1}^n
\mathbb E_{\pi\sim Q}\1\{S_i\subseteq B_\pi\},
\qquad
G(Q,P)
:=
\mathbb E_{\pi\sim Q}G(\pi,P).
\]
For success probabilities $a,b\in[0,1]$, their binary KL divergence is given by
\[
\bkl(a\Vert b)
=
a\log\frac ab+(1-a)\log\frac{1-a}{1-b}
\]
where $0\log(0/b)$ and $0\log[0/(1-b)]$ are defined as zero. A term is infinite when its numerator is positive and its denominator is zero.

\begin{theorem}[PAC-Bayes future-demand certificate]\label{thm:pacweak}
Let $n\geq2$ and let $\delta\in(0,1)$ be the allowed failure probability. With probability at least $1-\delta$, simultaneously for every posterior $Q$ with finite $\operatorname{KL}(Q\Vert Q_0)$,
\[
\bkl\!\left(\widehat G_n(Q)\middle\Vert G(Q,P)\right)
\leq
\frac{\operatorname{KL}(Q\Vert Q_0)+\log(2\sqrt n/\delta)}{n}.
\]
\end{theorem}

This is PAC-Bayes-kl applied to the bounded loss $1-\1\{S\subseteq B_\pi\}$ \citep{seeger2002pac,maurer2004note,chugg2023unified}. The theorem holds for arbitrary independent future-demand draws. The prior may depend on the fixed child task because the PAC sample consists of later demand experiments. Independence from those later demands is required. The same proof applies to any candidate set fixed before the demand sample and any bounded survival score. The neural application uses this generic candidate-set form for frozen network-policy pairs.

\begin{samepage}
A second connection follows from the existing Stack Theory reference measure. For a measurable set $E$ and a measurable region $A$ with $0<\mu(A)<\infty$, the conditional measure is given by
\[
\mu(E\mid A):=\frac{\mu(E\cap A)}{\mu(A)}.
\]
For probability measures $Q$ and $\nu$, write $Q\ll\nu$ when $Q$ assigns zero mass to every set that has zero mass under $\nu$.
\end{samepage}

\begin{theorem}[Buffer information projection]\label{thm:iproject}
Assume $0<\mu(B_\pi)\leq\mu(U)<\infty$. Let $Q$ be any probability measure that places all its mass inside $B_\pi$ and assigns zero mass to every set with zero $\mu(\,\cdot\,\mid U)$ measure.
\[
\operatorname{KL}\!\left(Q\middle\Vert\mu(\,\cdot\,\mid U)\right)
=
\operatorname{KL}\!\left(Q\middle\Vert\mu(\,\cdot\,\mid B_\pi)\right)
+
\log\frac{\mu(U)}{\mu(B_\pi)}.
\]
Consequently,
\[
\inf_{\substack{Q\ll\mu(\,\cdot\,\mid U)\\ Q(B_\pi)=1}}
\operatorname{KL}\!\left(Q\middle\Vert\mu(\,\cdot\,\mid U)\right)
=
\log\frac{\mu(U)}{\mu(B_\pi)},
\]
and the unique minimiser, except on sets of $\mu$-measure zero, is $\mu(\,\cdot\,\mid B_\pi)$.
\end{theorem}

Information projection means finding the distribution closest to a reference distribution in KL divergence while obeying a stated constraint.

The minimum information cost of restricting the unseen-region reference measure to one policy buffer is the negative logarithm of its normalised buffer measure. The theorem concerns completion measure and future compatibility. Appendix~\ref{app:pacproof} gives a completion-level survival certificate which scores each completion by its own future compatibility.

For the neural bank, treat each frozen network-policy pair as one candidate and each path as one bounded demand experiment. The uniform prior over the 100 frozen networks is chosen independently of path sampling. Because the theorem holds for every posterior at once, bank score may select one point posterior after paths are observed. Table~\ref{tab:pac} reports lower bounds for the declared external path-sampling rule.

\section{Related Work}\label{sec:related}

Research on generalisation often asks how much a learner has committed beyond the evidence. The disagreement is about where that commitment lives. Flatness places it in parameter-space geometry. Capacity measures place it in a hypothesis class. Minimum description length places it in an encoding. Stability places it in the response of a learning algorithm to changes in data. PAC-Bayes places it in a distribution over predictors and its information cost relative to a prior. Weakness places it in the completion structure of one correct policy inside an embodied language. These quantities can correlate because training, architecture, and vocabulary connect them. They remain different mathematical objects.

\subsection{Flatness and the level of explanation}

The flat-minima account began with the proposal that broad low-loss regions support generalisation \citep{hochreiter1997flat}. Later work connected large-batch training with high curvature and a generalisation gap, while Sharpness-Aware Minimisation turned the idea into an effective training objective \citep{keskar2017,foret2021sharpnessaware}. This line has practical force. It also raises a question about the level at which an explanation is meant to hold.

Raw Hessian curvature belongs to a parameterisation. ReLU symmetries can move a fixed input-output function along a family of parameter vectors with different curvature \citep{dinh2017sharp}. The rescaling experiment in Section~\ref{sec:curvature} reproduces that point without changing a prediction. This rules out raw Hessian trace as an intrinsic property of the represented function. It leaves room for two other claims. Curvature can record how an optimiser reached a parameter vector, and a curvature measure can be modified so that a stated class of reparameterisations leaves it unchanged. Normalised flatness, relative flatness, and adaptive sharpness pursue the second route \citep{tsuzuku2020normalized,petzka2021relative,kwon2021asam}. Their invariance groups, bounds, and optimisation uses deserve evaluation on their own terms.

Weakness addresses another level. Once an embodied language and a child task are fixed, a correct policy is scored by the compatible commitments it leaves available. The resulting quantity concerns behaviour expressible in that language. A parameter-space measure can track the same ordering in a training regime, yet such a correlation would make it a proxy for the behavioural structure rather than identify that structure by itself. This paper therefore compares raw curvature, relative flatness, margin, norm, and joint completion on the same network cohorts instead of treating one failed curvature definition as a rejection of every geometry-based account.

\subsection{Version spaces, capacity, and stability}

Version-space learning starts from the hypotheses consistent with observed data \citep{mitchell1982}. Stack Theory uses the same first step. The child-correct policy version space is $\Pi_\alpha$. Free-maxing then ranks the members of $\Pi_\alpha$ by weakness and chooses a maximiser. Earlier papers called this meta-approach \emph{w-maxing} because it maximises the symbol $w$. The name \emph{free-maxing} states the interpretation before the notation. A weakest correct policy is one that maximises freedom within the bounds of correctness. It maintains the integrity of a system without imposing unnecessary constraints.

The size $|\Pi_\alpha|$ and the weakness $w(\pi)$ answer different questions. The first counts how many child-correct policies remain available to the learner. The second scores one of those policies by the completion mass still compatible with choosing it. A large version space can contain policies with very different weaknesses, and every member receives the same class-level value $|\Pi_\alpha|$. Weakness is therefore a within-version-space ordering.

VC dimension also concerns a class rather than one selected policy. It asks for the largest sample that a hypothesis class can label in every possible way \citep{vapnik1971}. This supports worst-case generalisation bounds and comparisons between classes. It does not normally distinguish two child-correct policies inside one fixed class. Norm and margin bounds refine class-level capacity using properties of the selected predictor and sample \citep{neyshabur2017exploring,bartlett2017spectral}. Spectral-product margin performs well in the present cohorts and has larger point estimates than $J_H$. Its mathematical object is observed class separation normalised by network scale. Weakness instead measures which further commitments remain jointly satisfiable. A margin can serve as a low-cost indicator of that freedom without becoming identical to it.

Algorithmic stability moves the unit of analysis from a class or policy to the learning map from samples to predictors. A stable algorithm changes little when one training example is replaced, which supports bounds on its expected generalisation gap \citep{bousquet2002stability}. Weakness holds the child task fixed and compares policies that are all correct for it. The two views can be combined. A training algorithm induces a distribution over $\Pi_\alpha$, while weakness evaluates the policies it selects. Generalisation may then be discussed either conditionally for one policy, task, and demand law, or in expectation over the samples and randomness of an algorithm. The policy-level and algorithm-level questions should not be conflated.

\subsection{Simplicity, vocabularies, and the degeneracy trap}

Minimum description length and Solomonoff-style induction prefer short descriptions \citep{rissanen1978,grunwald2007,hutter2005}. These methods require a code or reference machine. The choice is often productive because a well-designed code carries inductive bias. It also means that description length belongs to an encoding. Stack Theory makes the corresponding choice explicit as a vocabulary $\mathfrak v$ and the satisfiable language $L_{\mathfrak v}$ it generates \citep{bennett2024b,bennett2026wrong}.

The distinction between $L_{\mathfrak v}$ and the full powerset $2^{\mathfrak v}$ is central. A member of $2^{\mathfrak v}$ is any bundle of vocabulary items. A member of $L_{\mathfrak v}$ is a bundle that can be true in at least one state. Physical incompatibilities, logical exclusions, and embodied limits remove many raw bundles in a structured vocabulary. Weakness counts only satisfiable supersets.

Stack Theory calls a vocabulary degenerate when all of its programs can be true together. For finite $\mathfrak v$, this is equivalent to
\[
L_{\mathfrak v}=2^{\mathfrak v}.
\]
Every subset is then satisfiable. For any statement $l$,
\[
w(l)=|\Ext{l}|=2^{|\mathfrak v|-|l|}.
\]
Literal-count simplicity and weakness consequently induce the same ranking. Each unasserted vocabulary item can be added or omitted independently, so fewer asserted items mean more completions. This is the degeneracy trap. It makes a simplicity rule appear fundamental because the vocabulary has removed every source of disagreement between simplicity and weakness.

A nondegenerate vocabulary contains incompatibility structure, so $L_{\mathfrak v}$ is a proper subset of $2^{\mathfrak v}$. Equal-length statements can then have different numbers of satisfiable supersets. A longer statement can also leave more compatible completions than a shorter one when its programs occupy a less restrictive part of the language. Description length and weakness may still correlate. Architecture, usage frequency, and learned representation can make frequently successful weak policies cheaper to encode. The correlation then comes from how the vocabulary was built and used. It is not an identity between short form and adaptable function.

This distinction also clarifies the scope of claims about AIXI and Ockham-style priors. A code can be an effective engineering prior. It does not provide a representation-independent ordering of embodied policies. Weakness retains the vocabulary dependence that any embodied criterion must have, while remaining unchanged by reparameterisations that preserve represented behaviour within the declared language. Across embodiments, semantic Unagi compares matched external truth conditions rather than raw code lengths or raw extension counts \citep{bennett2026wrong}.

\subsection{Maximum entropy and maximum freedom}

Maximum entropy is often described as least commitment in probability. Given a finite sample space and stated constraints, it chooses the distribution with greatest Shannon entropy \citep{jaynes1957}. Weakness is also a least-commitment principle, but it operates on policies and feasible completions. Their relation depends on the sample space over which entropy is maximised.

Let $Y$ be drawn uniformly from the feasible embodied language $L_{\mathfrak v}$. This is the maximum-entropy distribution on that finite sample space when no further information is imposed. For any statement $l$,
\[
P(l\subseteq Y)=\frac{|\Ext{l}|}{|L_{\mathfrak v}|}=\frac{w(l)}{|L_{\mathfrak v}|}.
\]
Thus normalised weakness is exactly the maximum-entropy continuation probability once the sample space consists of feasible embodied statements. The same calculation applies to future demands. If $S$ is uniform on $2^U$, then
\[
P(S\subseteq B_\pi)=2^{|B_\pi|-|U|},
\]
so the weakest correct policy is Bayes optimal under maximum entropy over the declared demand-set space.

A different result follows when entropy is maximised over raw vocabulary bits. Let $Z$ be uniform on $2^{\mathfrak v}$. Every vocabulary item is then included independently with probability one half, and
\[
P(l\subseteq Z)=2^{-|l|}.
\]
This is a literal-count rule. It agrees with weakness throughout the degenerate case $L_{\mathfrak v}=2^{\mathfrak v}$. In a nondegenerate language, $Z$ assigns probability to impossible bundles. Conditioning on feasibility changes the distribution from independent vocabulary bits to the uniform distribution on $L_{\mathfrak v}$, and the continuation probability becomes proportional to weakness. The apparent conflict between maximum entropy and weakness is therefore a conflict between sample spaces. Maximum entropy over feasible embodied continuations recovers weakness. The unconstrained calculation always yields literal count, but it coincides with maximum entropy over embodied continuations only when syntax and feasibility coincide.

Exchangeability generalises the uniform calculation. Demands of equal size receive equal treatment, so larger buffers contain at least as much probability mass. Known unequal demand rates induce weighted weakness, as shown in Section~\ref{sec:theory}. Maximum entropy is one demand model within this family. Weakness is the structural quantity that remains after the demand law and embodied language have been declared.

\subsection{PAC-Bayes, function-space counts, and empirical comparison}

PAC-Bayes bounds expected loss of a posterior over predictors using empirical performance and a KL penalty relative to a prior \citep{mcallester1999pac,seeger2002pac,maurer2004note}. Nonvacuous neural bounds show that this framework can carry quantitative content even in heavily parameterised models \citep{dziugaite2017computing}. PAC-Bayes and weakness intersect in two ways here. Theorem~\ref{thm:pacweak} applies a standard PAC-Bayes-kl inequality to future-demand survival. Theorem~\ref{thm:iproject} then shows that restricting a reference measure on the unseen region to one policy buffer has minimum KL cost $\log[\mu(U)/\mu(B_\pi)]$. Greater buffer measure lowers that particular information cost.

This connection has a defined boundary. PAC-Bayes permits many priors, posteriors, hypothesis spaces, and losses. A parameter-space posterior built from local perturbations can express a flatness argument. A posterior over child-correct policies under demand-survival loss expresses a weakness argument. The present theorems establish the latter and do not reduce every PAC-Bayes analysis to weakness.

Activation-region counts and the joint bank occupy another neighbouring area. ReLU region counts measure how a network partitions input space and have been studied as measures of expressivity \citep{hanin2019regions}. Counting regions independently can ignore the fact that one set of classifier weights must satisfy all commitments at once. The statistic $J_H$ keeps that shared-head constraint by testing bundles of probe labels with one affine classifier. It is exact for the finite indexed bank and invariant under invertible affine changes of feature coordinates. Formal weakness additionally requires a common task, one policy version space, and a measure over the relevant extension. Proposition~\ref{prop:common} states the exact compatibility condition for the anchor-augmented policies. Until that condition and the shared task reading are established, $J_H$ is best read as matched external-demand survival in separate feature languages.

Comparative studies find that rankings induced by generalisation measures can change with architecture, optimiser, dataset, and experimental control \citep{jiang2020fantastic}. The present experiments therefore use the same networks for every measure, orient each score before comparison, correct the declared test families, and report paired differences between correlations. The random-label cohort adds a separate test motivated by the ability of neural networks to interpolate arbitrary labels \citep{zhang2017rethinking}. Its result is informative for the empirical bank. Uniform arbitrary-label freedom rises, yet its relation to held-out accuracy disappears. The theory consequently requires the task, language, and demand law to remain explicit. Weakness is universal as a selection rule within those declared objects. A single raw count across unrelated tasks is not.

\section{Discussion}\label{sec:discussion}

Stack Theory makes a universal claim about the selection rule. For a fixed task, embodied language, and demand law, generalisation is compatible future-demand mass. Exchangeable ignorance gives ordinary weakness, independent unequal rates give weighted weakness, and continuous languages use a reference measure. The numerical value is body- and law-relative, while weakest-correct selection persists across those choices.

The matched bank asks a finite neural question. Every network faces the same external probe-label demands. Its representation and empirical statement determine which demands one affine head can satisfy. Thus $J_H$ is exact for the indexed bank and invariant under invertible affine feature changes. The bridge is given by Proposition~\ref{prop:common}. Shared anchor labels create one compatible external statement. A direct semantic-Unagi test must also show that each network-specific task is a reading of one shared task template and that each empirical policy is correct for its reading.

On natural labels, $J_H$ predicts accuracy. Spectral-product margin has larger point estimates. Random-label training yields greater uniform-bank feasibility and a score-accuracy association centred near zero. The bank therefore measures affine-readout freedom under its declared probe law. The result supports a task-specific interpretation.

AFMaI fixes the penultimate representation, 32 anchors, 512 sampled statements, four equal-weight orders, and one probe-label law before test access. The cohorts use low-data multilayer perceptrons on two image datasets. Almost no demands containing 16 simultaneous probe labels are feasible, and linear-programming cost grows with feature dimension, class count, empirical-policy size, and bank size. The sequel calibrates a family of weakness measurements on separate cohorts and uses the frozen measurement rule for model selection or evolutionary intervention.

PAC-Bayes certifies expected future-demand survival under sampled demands. Information projection identifies the KL cost of concentrating the unseen-region measure on a buffer. Raw Hessian trace varies across parameterisations of one function. Curvature conditioned on coordinates, optimiser, and trajectory remains an informative diagnostic, alongside relative flatness, margin, and stability.

Raw Hessian flatness is a coordinate account. Generalisation is the survival of commitments under future demands.

\acks{This study was self-funded.}
\clearpage
\appendix

\section{Proofs}\label{app:proofs}

\begin{proof}[Proof of Theorem~\ref{thm:dom}]
The result follows from Definition~\ref{def:meas}. Correct policies satisfy
\[
\Ext{\pi}=O_\alpha\sqcup B_\pi,
\]
so finite $\mu(O_\alpha)$ makes $\mu$-weakness and buffer measure induce the same ordering. The nondecreasing survival function preserves that ordering. The strict part is given by strict increase on attained buffer measures.
\end{proof}

\begin{proof}[Equal-measure symmetry implies $\mu$-exchangeability]
Equal-measure symmetry makes $P(S\subseteq B)$ a function of $\mu(B)$. Let $\mu(A)\leq\mu(B)$. Atomlessness ensures that a measurable $B'\subseteq B$ exists with $\mu(B')=\mu(A)$. Therefore
\[
P(S\subseteq A)=P(S\subseteq B')\leq P(S\subseteq B),
\]
where the inequality follows from
\[
\{S\subseteq B'\}\subseteq\{S\subseteq B\}.
\]
Thus the function of measure is nondecreasing.
\end{proof}

\begin{proof}[Proof of Proposition~\ref{prop:common}]
For the first part, choose $y\in O_\alpha$. Correctness gives
\[
y\in\Ext{I_\alpha}\cap\Ext{\pi_i}
\]
for every $i$. Hence $\pi_i\subseteq y$ for every $i$, so $\bigcup_i\pi_i\subseteq y$. Since $y\in L_{\mathfrak v}$, the union is satisfiable and belongs to $L_{\mathfrak v}$.

For the converse, put $\bar\pi:=\bigcup_i\pi_i\in L_{\mathfrak v}$ and define
\[
I_\alpha:=\{\bar\pi\},
\qquad
O_\alpha:=\Ext{\bar\pi}.
\]
Then $O_\alpha\neq\emptyset$ because $\bar\pi\in\Ext{\bar\pi}$. Also $\pi_i\subseteq\bar\pi$, so $\Ext{\bar\pi}\subseteq\Ext{\pi_i}$. Therefore
\[
\Ext{I_\alpha}\cap\Ext{\pi_i}
=
\Ext{\bar\pi}\cap\Ext{\pi_i}
=
\Ext{\bar\pi}
=
O_\alpha
\]
for every $i$.

\end{proof}

\begin{proof}[Proof of Corollary~\ref{cor:anchor}]
In the external label language, assignments on different inputs have a common total extension, while two class assignments for one input are disjoint. The union of the anchor-augmented policies is therefore satisfiable exactly when every network assigns the same class at every anchor. Apply Proposition~\ref{prop:common}.
\end{proof}

\begin{proof}[Proof of Proposition~\ref{prop:bank}]
Draw an index $J$ uniformly from $\{1,\ldots,m\}$. Then
\[
P(\text{success})
=
\frac1m\sum_{j=1}^{m}
\1\{\pi_f\cup s_j^{(f)}\in L_{\mathfrak v_{\varphi_f}}\}
=
J_H(f).
\]
For the implemented design, write the bank average as the average of 128 path scores, each being the mean of the four fixed nested orders. The paths are independent and identically distributed under the declared path-sampling rule, so linearity of expectation shows that its expected value equals the expected path score.
\end{proof}

\begin{proof}[Proof of Theorem~\ref{thm:affine}]
For every input $x$,
\[
W'\varphi'_f(x)+b'
=
WA^{-1}(A\varphi_f(x)+a)+b-WA^{-1}a
=
W\varphi_f(x)+b.
\]
Every class-score difference is preserved. The head map is one-to-one and onto because $A$ is invertible. Every joint statement therefore has the same feasibility status.
\end{proof}

\section{PAC-Bayes Proofs}\label{app:pacproof}

\begin{proof}[Proof of Theorem~\ref{thm:pacweak}]
Apply PAC-Bayes-kl to
\[
\ell(\pi,S)=1-\1\{S\subseteq B_\pi\}.
\]
The resulting inequality is
\[
\bkl\!\left(1-\widehat G_n(Q)\middle\Vert1-G(Q,P)\right)
\leq
\frac{\operatorname{KL}(Q\Vert Q_0)+\log(2\sqrt n/\delta)}{n}.
\]
Use $\bkl(1-a\Vert1-b)=\bkl(a\Vert b)$.
\end{proof}

\begin{proof}[Proof of Theorem~\ref{thm:iproject}]
Here $dQ/d\nu$ denotes the density of measure $Q$ relative to measure $\nu$. On $B_\pi$,
\[
\frac{d\mu(\,\cdot\,\mid B_\pi)}
     {d\mu(\,\cdot\,\mid U)}
=
\frac{\mu(U)}{\mu(B_\pi)}.
\]
For every $Q$ supported on $B_\pi$,
\[
\log\frac{dQ}{d\mu(\,\cdot\,\mid U)}
=
\log\frac{dQ}{d\mu(\,\cdot\,\mid B_\pi)}
+
\log\frac{\mu(U)}{\mu(B_\pi)}.
\]
The identity is given by integrating with respect to $Q$. The infimum and equality condition follow from nonnegativity of KL.
\end{proof}

\begin{theorem}[Completion-level survival certificate]\label{thm:compcert}
Assume $0<\mu(B_\pi)\leq\mu(U)<\infty$, $n\geq2$, and $\delta\in(0,1)$. Let $S_1,\ldots,S_n$ be independent future-demand sets drawn from $P$. Draw completions from $\mu(\,\cdot\,\mid B_\pi)$. For each sampled demand $S_i$, score completion $y$ by $\1\{S_i\subseteq\Ext{y}\cap U\}$. Write
\[
\widehat C_n(\pi)
:=
\frac1n\sum_{i=1}^{n}
\int_{B_\pi}\1\{S_i\subseteq\Ext{y}\cap U\}\,d\mu(y\mid B_\pi)
\]
and
\[
C(\pi,P)
:=
\int_{B_\pi}P(S\subseteq\Ext{y}\cap U)\,d\mu(y\mid B_\pi).
\]
With probability at least $1-\delta$,
\[
\bkl\!\left(\widehat C_n(\pi)\middle\Vert C(\pi,P)\right)
\leq
\frac{\log[\mu(U)/\mu(B_\pi)]+\log(2\sqrt n/\delta)}{n}.
\]
Moreover $C(\pi,P)\leq G(\pi,P)$.
\end{theorem}

\begin{proof}
Apply PAC-Bayes-kl on candidate space $U$ with prior $\mu(\,\cdot\,\mid U)$, posterior $\mu(\,\cdot\,\mid B_\pi)$, and the stated completion loss. The KL term is given by Theorem~\ref{thm:iproject}. If $y\in B_\pi$, then $y\in\Ext{\pi}$ and $\Ext{y}\subseteq\Ext{\pi}$. Completion survival therefore implies policy survival.
\end{proof}

\begin{proposition}[Pushforward invariance]\label{prop:pacinv}
Let $U'$ be another unseen region with measure $\mu'$, policy $\pi'$, and buffer $B_{\pi'}$. Let $h:U\to U'$ be a reversible relabelling for which both $h$ and $h^{-1}$ preserve measurable sets. The notation $h_\#\mu=\mu'$ means that relabelling transports measure $\mu$ to $\mu'$. Let $P':=h_\#P$ be the law of the relabelled demand set $h(S)$. Assume also that $h(B_\pi)=B_{\pi'}$. Then
\[
h_\#\mu(\,\cdot\,\mid U)=\mu'(\,\cdot\,\mid U'),
\qquad
h_\#\mu(\,\cdot\,\mid B_\pi)=\mu'(\,\cdot\,\mid B_{\pi'}),
\]
and the buffer measure and information-projection KL are unchanged. Policy survival satisfies $G(\pi,P)=G(\pi',P')$. Completion-level survival is also invariant when
\[
h(\Ext{y}\cap U)=\operatorname{Ext}'(h(y))\cap U'
\]
for every $y$ except a set of zero $\mu(\,\cdot\,\mid B_\pi)$ measure, where $\operatorname{Ext}'$ denotes extension in the target embodied language.
\end{proposition}

\begin{proof}
Conditioning a measure on a region commutes with a reversible measure-preserving relabelling. KL is invariant when both measures pass through the same relabelling. Relabelling preserves set containment. The additional extension condition is exactly what is needed for the completion-level event.
\end{proof}
\section{Linear Feasibility}\label{app:lp}

For feature vector $z\in\R^d$ and required class $y$, let $W_c$ denote row $c$ of the affine-head weight matrix. The affine-head condition is
\[
(W_y-W_c)^\top z+(b_y-b_c)\geq1
\qquad\text{for every }c\in\mathcal C\setminus\{y\}.
\]
Writing every entry of $W$ and $b$ as one parameter vector makes each inequality linear in $K(d+1)$ variables. One LP is formed by combining inequalities from the empirical policy and sampled commitments. For the submitted architecture, $K(d+1)=90$. The largest statement has $(250+32+16)(10-1)=2682$ inequalities.

The bank uses nested statements and permits early termination. The first contradictory prefix certifies every longer prefix on the same path, so the implementation reuses that certificate.

\section{Protocol}\label{app:protocol}

A fixed training-partition permutation allocates 250 child inputs, 32 anchors, and 128 free probes. The random-label cohort uses the same input indices as natural-label MNIST. Training uses stochastic gradient descent (SGD) with batches of 64 child examples. The learning rate, which sets update size, begins at $0.1$ and is multiplied by $0.3$ after epochs 500, 1000, and 1500. Momentum is $0.9$, so recent update directions are averaged, and weight decay is $0$, so no parameter-shrinkage penalty is applied. One epoch is one pass through the child set. A network must interpolate the child set, meaning classify every child input correctly, with minimum class-score gap at least $10^{-3}$ within 2000 epochs. The protocol preserves the original seed list and reports every failed seed.

The bank seed is 314159. Each of 128 paths samples 16 distinct free probes and 16 labels. Prefixes at orders $2,4,8,16$ yield 512 statements. The same bank is used for every network and cohort. The LP margin is one. Solver-unknown outcomes are recorded through lower and upper feasibility bounds. A cohort stops when any network exceeds one percent unknown statements.

Raw Hessian trace uses 16 Hutchinson samples, a random-vector estimator of matrix trace, on child cross-entropy loss. Cross-entropy penalises low probability assigned to the observed class. Final-layer relative flatness uses Definition 3 of \citet{petzka2021relative}. Let $w_s$ be row $s$ of the final-layer weight matrix, and let $H_{s,s'}$ be the Hessian block formed by second derivatives of child loss with respect to rows $s$ and $s'$. The statistic is
\[
\kappa_{\mathrm{Tr}}=\sum_{s,s'}\langle w_s,w_{s'}\rangle\operatorname{Tr}(H_{s,s'}),
\]
after rescaling every feature coordinate with nonzero sample variance so that its variance equals one. The spectral-product statistic is the tenth-percentile child score margin divided by the product of layer spectral norms. It uses the leading margin-over-spectral-product factor from \citet{bartlett2017spectral}. The full bound also contains a layerwise correction factor. Parameter norm is the Euclidean norm of every parameter.

Each cohort first writes 100 model-side rows containing training-side quantities. The updated notebooks retain each anchor-prediction vector and its SHA-256 digest. A SHA-256 digital fingerprint of all model-side rows is stored in \texttt{lock.json}. The next stage downloads the official test partition and adds accuracy. The reducer recomputes the digest and rejects changed rows.

Permutation $p$ values use 50,000 draws. Confidence intervals use 10,000 network bootstrap samples. Before correlation, raw Hessian trace, relative flatness, and weight norm are negated. The two natural joint tests form the primary Holm family. Eight baseline tests form another Holm family. Four within-dataset correlation differences use Bonferroni $98.75\%$ intervals. Odd and even paths produce split-bank estimates. The random-label association and paired natural-minus-random mean difference form a third Holm family. When raw rows are present, the reducer also reports a secondary partial-rank analysis. It converts $J_H$, accuracy, and spectral-product margin to ranks, removes the linear rank effect of margin from the first two, and correlates the remaining rank residuals.

\section{Generated Diagnostics}\label{app:diag}
\begin{table}[h]
\centering
\small
\caption{Paired differences between the joint-measure correlation and each oriented baseline correlation. Intervals are Bonferroni $98.75\%$ intervals within each dataset.}
\label{tab:cmpres}
\begin{tabular}{@{}llcc@{}}
\toprule
Dataset & Baseline & $\rho_J-\rho_b$ & $98.75\%$ CI \\
\midrule
MNIST & Raw Hessian trace & $+0.198$ & [$-0.149$, $+0.553$] \\
MNIST & Final-layer relative flatness & $+0.122$ & [$-0.201$, $+0.448$] \\
MNIST & Spectral-product margin & $-0.215$ & [$-0.467$, $+0.037$] \\
MNIST & Weight norm & $-0.108$ & [$-0.422$, $+0.194$] \\
Fashion-MNIST & Raw Hessian trace & $+0.412$ & [$+0.106$, $+0.701$] \\
Fashion-MNIST & Final-layer relative flatness & $+0.507$ & [$+0.178$, $+0.796$] \\
Fashion-MNIST & Spectral-product margin & $-0.077$ & [$-0.279$, $+0.126$] \\
Fashion-MNIST & Weight norm & $+0.107$ & [$-0.166$, $+0.384$] \\
\bottomrule
\end{tabular}
\end{table}

\begin{table}[h]
\centering
\small
\caption{Estimator diagnostics. The LP columns report correlations at the lower and upper feasibility bounds.}
\label{tab:diagres}
\begin{tabular}{@{}lcccc@{}}
\toprule
Dataset & $\rho_{\rm low}$ & $\rho_{\rm high}$ & Half-bank $\rho$ & Spearman--Brown \\
\midrule
MNIST & $+0.290$ & $+0.290$ & $+0.887$ & $0.940$ \\
Fashion-MNIST & $+0.468$ & $+0.468$ & $+0.915$ & $0.956$ \\
\bottomrule
\end{tabular}
\end{table}

\begin{table}[h]
\centering
\small
\caption{Mean feasible fraction by commitment order. Rates near zero or one show that an order distinguishes few networks.}
\label{tab:ordres}
\begin{tabular}{@{}lcccc@{}}
\toprule
Dataset & Order 2 & Order 4 & Order 8 & Order 16 \\
\midrule
MNIST & $0.671$ & $0.402$ & $0.073$ & $0.000$ \\
Fashion-MNIST & $0.741$ & $0.487$ & $0.133$ & $0.000$ \\
Random labels & $0.843$ & $0.628$ & $0.206$ & $0.003$ \\
\bottomrule
\end{tabular}
\end{table}

The order-16 component is nearly at the feasibility floor. Removing or reweighting it after seeing outcomes would change the declared measure, so the primary analysis retains it. Future calibration should take place on a separate cohort before the instrument is frozen.

\section{Rescaling Values}\label{app:rescale}

For $T_\beta$, the first hidden layer's weights and biases are multiplied by $\beta$ and the following layer's weights are divided by $\beta$. For $T_\gamma$, the second hidden layer's weights and biases are multiplied by $\gamma$ and the output layer's weights are divided by $\gamma$. Positive ReLU homogeneity makes both maps function preserving.

\begin{table}[h]
\centering
\caption{One network trained on 6000 points. Function-preserving rescaling changes raw Hessian trace while preserving every prediction.}
\label{tab:reparam}
\begin{tabular}{@{}llrr@{}}
\toprule
Map & Value & Test accuracy & Hessian trace \\
\midrule
$T_\beta$ & 1 & 0.9392 & 66.1 \\
$T_\beta$ & 2 & 0.9392 & 81.3 \\
$T_\beta$ & 5 & 0.9392 & 388.8 \\
$T_\beta$ & 10 & 0.9392 & 1607.1 \\
$T_\beta$ & 20 & 0.9392 & 6525.9 \\
$T_\gamma$ & 1 & 0.9392 & 62.9 \\
$T_\gamma$ & 5 & 0.9392 & 255.2 \\
$T_\gamma$ & 20 & 0.9392 & 3378.9 \\
\bottomrule
\end{tabular}
\end{table}

The two identity rows differ slightly because Hessian trace is estimated from random vectors. Across twelve retained networks, the largest ratio was $99\times$. The experiment archive contains the notebooks and recorded summary values. Checkpoint files remain in the run directory on Google Drive.

\section{Reproducibility}\label{app:repro}

The supplement contains three self-contained cohort notebooks, one reducer, an anchor-audit notebook, mathematical checks, requirements, protocol documents, generated tables, and aggregate results. Each cohort notebook requests an L4 graphics processing unit (GPU) and high-memory Colab runtime, mounts Google Drive, saves one model per seed, checkpoints LP status after every 32 new solves, and resumes by seed. The updated notebooks retain anchor predictions in future run rows. The audit notebook recovers them from saved checkpoints and existing LP records. The reducer checks hashes, splits, nested-bank structure, statistical families, LP bounds, and the PAC-Bayes path certificate.

The submitted archive includes the notebooks and aggregate results. The checkpoint directory and raw per-network rows remain in Google Drive. Recomputing row-level correlations requires restoring those files or rerunning the notebooks. The completed cohort awaits the common-anchor audit.

The source package contains the manuscript, JMLR style file, master bibliography, and PDF figures. Dataset files enter through official downloads during notebook execution. Test labels enter in the final evaluation cell. Dataset licences and attribution are documented in the supplement.

\bibliography{master_bibliography}
\end{document}